\documentclass{article} 
\usepackage{nips13submit_e,times}
\usepackage{hyperref}
\usepackage{url}
\usepackage{graphicx} 
\usepackage{subfigure}

\usepackage[ruled,vlined,commentsnumbered]{algorithm2e}
\usepackage{balance}

\usepackage{hyperref}
\usepackage{amsmath}
\usepackage{amssymb}
\usepackage{amsthm}
\usepackage{dsfont}
\usepackage{multirow}

\DeclareMathOperator*{\argmin}{arg\,min}


\newtheorem{thm}{Theorem}

\newtheorem{lemma}{Lemma}
\theoremstyle{definition}
\newtheorem{mydef}{Definition}

\newcommand{\rmnum}[1]{\romannumeral #1}
\newcommand{\Rmnum}[1]{\expandafter\@slowromancap\romannumeral #1@}
\newcommand{\ie}{{\it i.e. }}
\newcommand{\eg}{{\it e.g. }}
\newcommand{\wrt}{{\it w.r.t. }}

\title{RAPID: Rapidly Accelerated Proximal Gradient Algorithms for Convex Minimization}

\author{
Ziming Zhang \quad\quad Venkatesh Saligrama \\
Department of Electrical and Computer Engineering\\
Boston University, Boston, MA 02215 \\
\texttt{\{zzhang14, srv\}@bu.edu} \\
}

\nipsfinalcopy 

\begin{document}

\maketitle

\begin{abstract}
In this paper, we propose a new algorithm to speed-up the convergence of accelerated proximal gradient (APG) methods. In order to minimize a convex function $f(\mathbf{x})$, our algorithm introduces a simple line search step after each proximal gradient step in APG so that a biconvex function $f(\theta\mathbf{x})$ is minimized over scalar variable $\theta>0$ while fixing variable $\mathbf{x}$. We propose two new ways of constructing the auxiliary variables in APG based on the intermediate solutions of the proximal gradient and the line search steps. We prove that at arbitrary iteration step $t (t\geq1)$, our algorithm can achieve a smaller upper-bound for the gap between the current and optimal objective values than those in the traditional APG methods such as FISTA \cite{Beck:2009:FIS:1658360.1658364}, making it converge faster in practice. In fact, our algorithm can be potentially applied to many important convex optimization problems, such as sparse linear regression and kernel SVMs. Our experimental results clearly demonstrate that our algorithm converges faster than APG in all of the applications above, even comparable to some sophisticated solvers.
\end{abstract}

\section{Introduction}\label{sec:intr}
As a general convex minimization algorithm, accelerated proximal gradient (APG) has been attracting more and more attention recently, and it has been widely used in many different research areas such as signal processing \cite{Combettes_Pesquet}, computer vision \cite{conf/cvpr/BaoWLJ12}, and data mining \cite{Zhou:2010:NFG:1933307.1934518}. In general, APG solves the following problem:
\begin{equation}\label{eqn:problem}
\min_{\mathbf{x}\in\mathcal{X}}f(\mathbf{x})=f_1(\mathbf{x})+f_2(\mathbf{x})
\end{equation}
where $\mathcal{X}\subseteq\mathbb{R}^n$ denotes the closed and convex feasible set for variable $\mathbf{x}\in\mathbb{R}^n$, and $f(\mathbf{x}):\mathbb{R}^n\rightarrow\mathbb{R}$ is a convex function, which consists of a convex and differentiable function $f_1$ with Lipschitz constant $L\geq0$ and a convex but non-differentiable function $f_2$. 

In APG, proximal gradient is used to update variables based on the proximity operator, denoted as $\mathbf{prox}$. The basic idea of proximity operator is to approximate a convex function using a strongly convex function whose minimizer in the feasible set is returned as an approximate solution for the original minimization problem. At the optimal solution, the solution returned by proximity operator is identical to the optimal. As an example among classic APG algorithms, the basic version of FISTA \cite{Beck:2009:FIS:1658360.1658364} is shown in Alg. \ref{alg:fista}, where $\forall t=1,\cdots,T, \gamma_t$ denotes the step size. From FISTA, we can see that APG generates an {\it auxiliary variable} (\ie $\mathbf{v}_t$ in Alg. \ref{alg:fista}) for proximal gradient so that the convergence rate of APG for general convex optimization is $O(\frac{1}{T^2})$, which was proved to be optimal for first-order gradient descent methods \cite{tagkey198475}. 

Generally speaking, the computational bottleneck in APG comes from the following two aspects:

\textbf{(\rmnum{1}) Computation of proximal gradients.}
Evaluating the gradients of $f_1$ could be time-consuming, because the evaluation is over the entire dataset. This situation is more prominent for high dimensional and large-scale datasets. Also, projecting a point into the feasible set may be difficult. Many recent approaches have attempted to reduce this computational complexity. Inexact proximal gradient methods \cite{DBLP:conf/nips/SchmidtRB11} allow to approximate the proximal gradients with controllable errors in a faster way while guaranteeing the convergence. Stochastic proximal gradient methods \cite{atchade2014stochastic,rosasco2014convergence} allow to compute the proximal gradients using a small set of data in a stochastic fashion while guaranteeing the convergence as well. Distributed proximal gradient methods \cite{chen2012fast} decompose the optimization problem into sub-problems and solve these sub-problems locally in a distributed way using proximal gradient methods.

\textbf{(\rmnum{2}) Number of iterations.} 
In order to minimize the number of iterations, intuitively in each iteration the resulting function value should be as close to the global minimum as possible. One way to achieve this is to optimize the step size in the proximal gradient (\eg $\gamma_t$ in Alg. \ref{alg:fista}), which unfortunately may be very difficult for many convex functions. Instead, in practice line search \cite{Beck:2009:FIS:1658360.1658364,DBLP:conf/icml/LinX14,DBLP:journals/focm/ScheinbergGB14,DBLP:journals/siamjo/Xiao013} is widely used to estimate the step size so that the function value is decreasing. For instance, backtracking \cite{Beck:2009:FIS:1658360.1658364,DBLP:journals/focm/ScheinbergGB14} is a common line search technique to tune the step size gradually. In general, the line search in the proximal gradient step has to evaluate the function repeatedly by changing the step size so that numerically the learned step size is close to the optimal. Alternatively many types of restarting schemes \cite{DBLP:conf/icml/LinX14,DBLP:journals/mp/Nesterov13,Brendan2013} have been utilized to reduce the number of iterations empirically.  Here additional restarting conditions are established and evaluated periodically. If such conditions are satisfied, the algorithm will be re-initialized using current solutions.

\begin{algorithm}[t]\small
\SetAlgoLined
\SetKwInOut{Input}{Input}\SetKwInOut{Output}{Output}
\Input{$f(\mathbf{x})$, $\lambda$, $\mathbf{x}_0$}
\Output{$\mathbf{x}$}
\BlankLine
\For{$t=1,\cdots,T$}{
$\mathbf{v}_t\leftarrow\mathbf{x}_t+\frac{t-1}{t+2}\left(\mathbf{x}_t-\mathbf{x}_{t-1}\right)$; \quad
$\mathbf{x}_t\leftarrow\mathbf{prox}_{\lambda f_2}(\mathbf{v}_t-\gamma_t\triangledown f_1(\mathbf{v}_t))$;
}
\Return $\mathbf{x}_T$\;
\caption{FISTA \cite{Beck:2009:FIS:1658360.1658364}: Fast Iterative Shrinkage-Thresholding Algorithm}\label{alg:fista}
\end{algorithm}

\subsection{Our Contributions}
In this paper, we focus on reducing the number of iterations, and simply assume that the non-differentiable function $f_2$ is {\it simple} \cite{DBLP:journals/mp/Nesterov13} for performing proximal gradient efficiently, \eg $\ell_1$ norm. 

Our first contribution is to propose a new general algorithm, {\it \underline{R}apidly \underline{A}ccelerated \underline{P}rox\underline{i}mal Gra\underline{d}ient (\textbf{RAPID})}, to speed up the empirical convergence of APG, where an additional simple line search step is introduced after the proximal gradient step. Fig. \ref{fig:intuition}(a) illustrates the basic idea of our algorithm in 2D. After the proximal gradient step, another line search is applied along the direction of the current solution $\mathbf{x}$. Ideally, we would like to find a scalar $\theta>0$ so that $\theta=\argmin_{\{\hat{\theta}|\hat{\theta}\mathbf{x}\in\mathcal{X}\}}f(\hat{\theta}\mathbf{x})$. Therefore, we can guarantee $f(\theta\mathbf{x})\leq f(\mathbf{x})$. The positiveness of $\theta$ guarantees that both $\mathbf{x}$ and $\theta\mathbf{x}$ point to the same direction so that the information from gradient descent can be preserved. Geometrically, this additional line search tries to push $\mathbf{x}$ towards the optimal solution, making the distance between the current and optimal solutions smaller. Unlike the line search in the proximal gradient, the computation of finding the optimal $\theta$ can be very cheap (\eg with close-form solutions) for many convex optimization problems, such as LASSO \cite{tibshirani96regression} and group LASSO \cite{yuan2006model} (see Section \ref{sec:app}). Also, in order to guarantee the convergence of our algorithm, we further propose two ways of constructing the auxiliary variables based on the intermediate solutions in the previous and current iterations, as illustrated in Fig. \ref{fig:intuition}(b).

Our second contribution is that theoretically we prove that at an arbitrary iteration $t$, the upper bound of the objective error $f(\theta_t\mathbf{x}_t)-\min_{\mathbf{x}\in\mathcal{X}}f(\mathbf{x})$ in our algorithm is consistently smaller than that of $f(\mathbf{x}_t)-\min_{\mathbf{x}\in\mathcal{X}}f(\mathbf{x})$ in traditional APG methods such as FISTA. This result implies that in order to achieve the same precision, the number of iterations in our algorithm is probably no more than that in APG. In other words, empirically our algorithm will converge faster than APG.

Our third contribution is that we apply our general algorithm to several interesting convex optimization problems, \ie LASSO, group LASSO, least square loss with trace norm \cite{Ji:2009:AGM:1553374.1553434,journals/mp/Tseng10}, and kernel support vector machines (SVMs), and compare our performance with APG and other existing solvers such as SLEP \cite{Liu:2009:SLEP:manual} and LIBSVM \cite{CC01a}. Our experimental results demonstrate the correctness of our theorems on faster convergence than APG, and surprisingly in most cases, our algorithm can be comparable with those sophisticated solvers.

This paper is organized as follows. In Section \ref{sec:alg}, we explain the details of our RAPID algorithm, including the new line search step and how to construct the auxiliary variables. In Section \ref{sec:app}, we take LASSO, group LASSO, least square loss with trace norm, and kernel SVMs as examples to demonstrate the empirical performance of our algorithm with experimental results and comparison with APG and other existing solvers. The theoretical results on the convergence rate of our algorithm are proven in Section \ref{sec:analysis}, and finally we conclude the paper in Section \ref{sec:con}.

\begin{figure}[t]
\begin{minipage}[b]{0.35\linewidth}
 \begin{center}
 \centerline{\includegraphics[width=0.9\columnwidth]{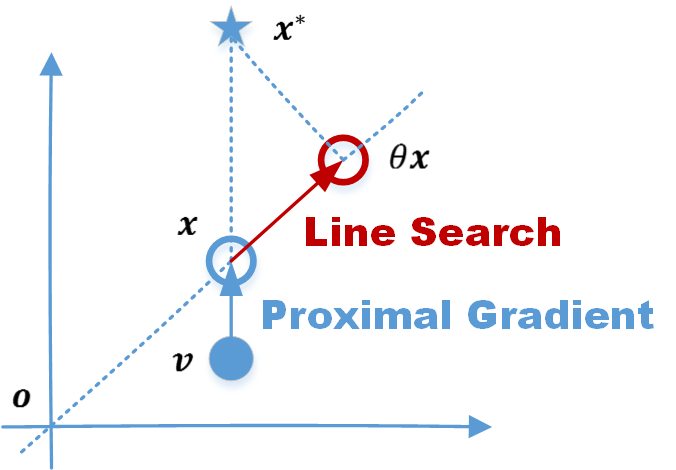}}
 \centerline{\footnotesize{(a)}}
 \end{center}
\end{minipage}
\begin{minipage}[b]{0.65\linewidth}
 \begin{center}
 \centerline{\includegraphics[width=\columnwidth]{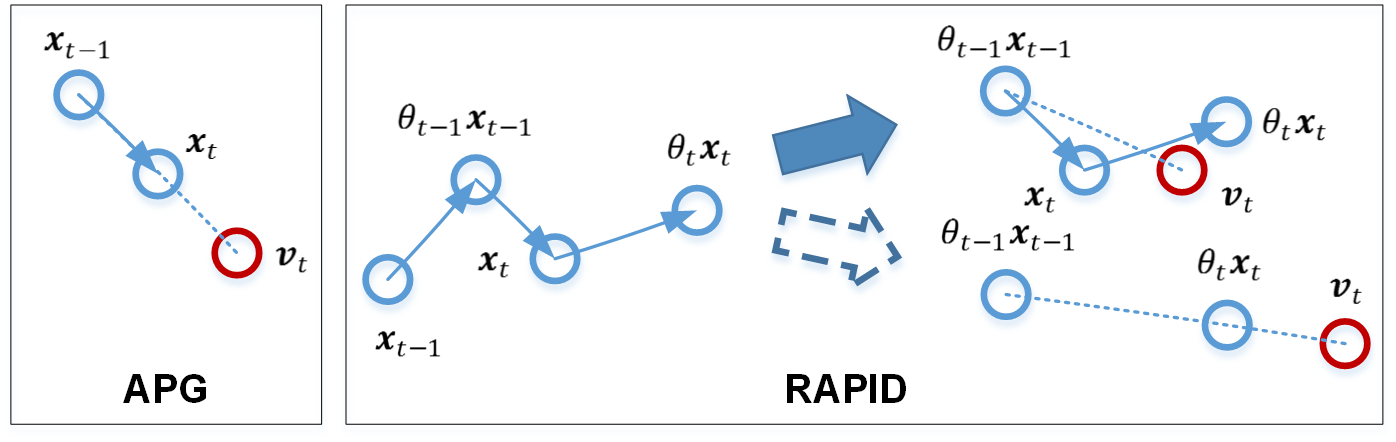}}
 \centerline{\footnotesize{(b)}}
 \end{center}
\end{minipage}\vspace{-5mm}
\caption{\footnotesize{(a) Illustration of the basic idea of our algorithm for improving convergence of APG in 2D, where $\mathbf{o}$ denotes the origin of the 2D coordinate system, $\mathbf{v}$, $\mathbf{x}$, $\theta\mathbf{x}$, and $\mathbf{x}^*$ denote the solutions at the initial point, after the proximal gradient step, after the line search step, and at the minimum, respectively, and the directed lines denote the updating directions. (b) Illustration of the differences between APG and our RAPID in terms of constructing the auxiliary variable $\mathbf{v}_t$ in iteration $t$, where directed lines denote the generating orders of intermediate solutions, and the dotted lines denote the directions of the auxiliary variable starting at the solution in iteration $t-1$. This figure is best viewed in color.}}\label{fig:intuition}
 \vspace{-2mm}
\end{figure}

\section{Algorithms}\label{sec:alg}

In general, there are two basic steps in each iteration in APG algorithms: (1) performing {\it proximal gradients}, and (2) constructing {\it auxiliary variables}. Proximal gradient is defined as applying {\it proximity operator} to a gradient descent step.

\begin{mydef}[Proximity Operator \cite{Combettes_Pesquet}]
The proximal operator $\mathbf{prox}:\mathbb{R}^n\rightarrow\mathbb{R}^n$ is defined by
\begin{equation}
\mathbf{prox}_{\lambda f}(\mathbf{v})=\argmin_{\mathbf{x}\in\mathcal{X}}\left(f(\mathbf{x})+\frac{1}{2\lambda}\|\mathbf{x}-\mathbf{v}\|_2^2\right),
\end{equation}
where $\|\cdot\|_2$ is the Euclidean norm, $\mathcal{X}$ is a closed and convex set, and $\lambda>0$ is a scaling parameter.
\end{mydef}

Alg. \ref{alg:rapid} shows our RAPID algorithm, where in each iteration $t(t\geq 1)$ four steps are involved:
\begin{itemize}
\item[Step 1:] A proximal gradient step using the auxiliary variable $\mathbf{v}_{t-1}$, same as APG.
\item[\textbf{\em Step 2}:] \textbf{\em A simple line search step along the direction of the current solution $\mathbf{x}_t$.} Actually the definition of $\theta_t$ in Alg. \ref{alg:rapid} is equivalent to the following equation:
\begin{equation}
\mathbf{x}_t^*=\mathbf{prox}_{\lambda_{\theta}f}(\mathbf{x}_t)\in\{\tilde{\mathbf{x}}_t|\exists\theta, \tilde{\mathbf{x}}_t=\theta\mathbf{x}_t\in\mathcal{X}\}.
\end{equation}
In other words, this line search step essentially {\it adapts} the current solution $\mathbf{x}_t$ to a better one in a very efficient way (\eg with close-form solutions). 
\item[Step 3:] Updating parameter $\eta_t$ used for constructing the new auxiliary variable $\mathbf{v}_t$, same as APG. Note that any number sequence $\{\eta_t\}$ can be used here as long as the sequence satisfies $\forall t\leq1, \frac{1-\eta_{t+1}}{\eta_{t+1}^2}\leq\frac{1}{\eta_t^2}$.
\item[\textbf{\em Step 4}:] \textbf{\em Updating the new auxiliary variable $\mathbf{v}_t$ using one of the following two equations:}
\begin{eqnarray}\label{eqn:rapid-1}
\mathbf{v}_t=\eta_t(1-\eta_{t-1}^{-1})\theta_{t-1}\mathbf{x}_{t-1}+\eta_t\eta_{t-1}^{-1}\mathbf{x}_t+(1-\eta_t)\theta_t\mathbf{x}_t,
\end{eqnarray}
\begin{eqnarray}\label{eqn:rapid-2}
\mathbf{v}_t=\eta_t(1-\eta_{t-1}^{-1})\theta_{t-1}\mathbf{x}_{t-1}+(1-\eta_t+\eta_t\eta_{t-1}^{-1})\theta_t\mathbf{x}_t.
\end{eqnarray}
In this way, our algorithm guarantees its convergence, but with different convergence rate. See our comparison results in Section \ref{sec:app}.
\end{itemize}
Fig. \ref{fig:intuition}(b) illustrates the differences in constructing the auxiliary variable between APG and our RAPID. In APG, the auxiliary variable $\mathbf{v}_t$ is constructed along the gradient of $\mathbf{x}_t-\mathbf{x}_{t-1}$ starting from $\mathbf{x}_t$. Similarly, in RAPID we would like to construct the auxiliary variable $\mathbf{v}_t$ using $\theta_t\mathbf{x}_t$ and the other intermediate solutions in the previous and current iterations (\ie $\mathbf{x}_{t-1}, \theta_{t-1}\mathbf{x}_{t-1}, \mathbf{x}_t$). It turns out that {\em all possible combinations of intermediate solutions for constructing $\mathbf{v}_t$ end up with Eq. \ref{eqn:rapid-1}}, with guaranteed better upper bounds over $f(\theta_t\mathbf{x}_t)-\min_{\mathbf{x}\in\mathcal{X}}f(\mathbf{x})$ than those over $f(\mathbf{x}_t)-\min_{\mathbf{x}\in\mathcal{X}}f(\mathbf{x})$ in APG in arbitrary iteration $t$ (see Theorem \ref{thm:rapid-1} in Section \ref{sec:analysis}). Under a mild condition, we can adopt the same way as APG to construct $\mathbf{v}_t$ using the final solutions in the previous and current iterations, \ie $\theta_{t-1}\mathbf{x}_{t-1}$ and $\theta_t\mathbf{x}_t$, which is exactly Eq. \ref{eqn:rapid-2}. However, for this setting we lose the theoretical guarantee of better upper bounds than APG, as shown in Theorem \ref{thm:rapid-2} in Section \ref{sec:analysis}. Nevertheless, surprisingly, in our experiments our algorithm using Eq. \ref{eqn:rapid-2} outperforms than that using Eq. \ref{eqn:rapid-1} with significant improvement in terms of empirical convergence (see Section \ref{sec:app} for details).

\begin{algorithm}[t]\small
\SetAlgoLined
\SetKwInOut{Input}{Input}\SetKwInOut{Output}{Output}
\Input{$f(\mathbf{x})$, $\lambda_{\mathbf{x}}$, $\lambda_{\theta}$}
\Output{$\mathbf{x}$}
\BlankLine
$\mathbf{x}_0\leftarrow\mathbf{0}$; $\mathbf{v}_0\leftarrow\mathbf{x}_0$; $\theta_0\leftarrow 1$; $\eta_0\leftarrow 1$;\\
\For{$t=1,\cdots,T$}{
$\mathbf{x}_t\leftarrow\mathbf{prox}_{\lambda_{\mathbf{x}}f_2}(\mathbf{v}_{t-1}-\gamma_t\triangledown f_1(\mathbf{v}_{t-1}))$;
$\theta_t\leftarrow\argmin_{\{\theta|\theta\mathbf{x}_t\in\mathcal{X}\}}\left\{f(\theta\mathbf{x}_t)+\frac{1}{2\lambda_{\theta}}\|\theta\mathbf{x}_t-\mathbf{x}_t\|_2^2\right\}$; \\
$\eta_t=\frac{\sqrt{\eta_{t-1}^4+4\eta_{t-1}^2}-\eta_{t-1}^2}{2}$; \quad Update $\mathbf{v}_t$ using either Eq. \ref{eqn:rapid-1} or Eq. \ref{eqn:rapid-2};
}
\Return $\theta_T\mathbf{x}_T$\;
\caption{RAPID: \underline{R}apidly \underline{A}ccelerated \underline{P}rox\underline{i}mal Gra\underline{d}ient Algorithms}\label{alg:rapid}
\end{algorithm}

\section{Numerical Examples}\label{sec:app}

In this section, we will explain how to apply our algorithm to solve (1) sparse linear regression (\ie LASSO, group LASSO, and least square fitting with trace-norm), and (2) binary kernel SVMs. We also compare our empirical performance with APG and some other existing solvers.

\subsection{Sparse Linear Regression}
\subsubsection{Problem Settings}
We denote $\mathbf{A}=\{\mathbf{a}_i\}_{i=1,\cdots,N}\in\mathbb{R}^{N\times d}$ as a data matrix, $\mathbf{y}=\{y_i\}_{i=1,\cdots,N}\in\mathbb{R}^{N}$ as a regression target vector, $\mathbf{Y}=\{\mathbf{y}_j\}_{j=1,\cdots,M}\in\mathbb{R}^{N\times M}$ as a matrix consisting of $M$ regression tasks, $\mathbf{x}\in\mathbb{R}^d$ as a linear regressor, $\mathbf{X}=\{\mathbf{x}_j\}_{j=1,\cdots,M}\in\mathbb{R}^{d\times M}$ as a matrix consisting of $M$ linear regressors, and $\lambda\geq0$ as a regularization parameter. As follows, for each method we list its loss function, regularizer, proximity operation, and optimal line search scalar $\theta$, which are used in Alg. \ref{alg:rapid}.

\textbf{(\rmnum{1}) Loss functions (\ie $f_1$, convex and differentiable).} Least square loss is used in all the methods, \ie $\frac{1}{2}\|\mathbf{A}\mathbf{x}-\mathbf{y}\|_2^2$ for LASSO and group LASSO, and $\frac{1}{2}\sum_{j=1}^{M}\|\mathbf{A}\mathbf{x}_j-\mathbf{y}_j\|_2^2$ for trace-norm.

\textbf{(\rmnum{2}) Regularizers (\ie $f_2$, convex but non-differentiable).} The corresponding regularizers in LASSO, group LASSO, and trace-norm are $\lambda\|\mathbf{x}\|_1$, $\lambda\sum_{g\in\mathcal{G}}\|\mathbf{x}^{(g)}\|_2$, and $\lambda\|\mathbf{X}\|_*=\mathbf{trace}\left(\sqrt{\mathbf{X}^T\mathbf{X}}\right)=\sum_{j=1}^{\min(d,M)}\sigma_j$, respectively. Here, $\|\cdot\|_1$ and $\|\cdot\|_2$ denote $\ell_1$ and $\ell_2$ norms, $(\cdot)^T$ denotes the matrix transpose operator, $g\in\mathcal{G}$ denotes a group index, $\mathbf{x}^{(g)}$ denotes a group of variables without overlaps, and $\sigma_j$ denotes the $j^{th}$ singular value for matrix $\sqrt{\mathbf{X}^T\mathbf{X}}$.

\textbf{(\rmnum{3}) Proximity operators.} According to their regularizers, the proximity operators can be calculated efficiently as follows, where $\mathbf{u}=\{u_j\}$ and $\mathbf{U}=\{\mathbf{u}_j\}$ are denoted as the variable vector and matrix after the gradient descent:
\begin{itemize}
\item LASSO: $\mathbf{prox}_{\lambda f_{\|\cdot\|_1}}(u_j)=\mathbf{sign}(u_j)\cdot\max\left\{0,|u_j|-\lambda\right\}$, where $\mathbf{sign}(u_j)=1$ if $u_j\geq0$, otherwise, $\mathbf{sign}(u_j)=-1$; and $|u_j|$ denotes its absolute value.
\item Group LASSO: $\mathbf{prox}_{\lambda f_{\|\cdot\|_2}}(u_j^{(g)})=\frac{u_j^{(g)}}{\|\mathbf{u}^{(g)}\|_2}\cdot\max\left\{0,\|\mathbf{u}^{(g)}\|_2-\lambda\right\}$.
\item Trace norm: Letting $\mathbf{U}=\mathbf{P}\cdot\mathbf{diag}(\boldsymbol{\sigma})\cdot\mathbf{Q}^T$, where $\mathbf{P}\in\mathbb{R}^{d\times n}$ and $\mathbf{P}\in\mathbb{R}^{M\times n}$ are two matrices and $\boldsymbol{\sigma}$ is a vector with $n$ singular values of $\mathbf{U}$, then we have $\mathbf{prox}_{\lambda f_{\|\cdot\|_*}}(\mathbf{U})=\mathbf{P}\cdot\mathbf{prox}_{\lambda f_{\|\cdot\|_1}}(\boldsymbol{\sigma})\cdot\mathbf{Q}^T$. Here $\mathbf{prox}_{\lambda f_{\|\cdot\|_1}}$ is an entry-wise operator.
\end{itemize}

\textbf{(\rmnum{4}) Optimal line search scalar $\theta$.} 
For each problem, we re-define $\theta$ as $\theta_t=\argmin_{\theta}\left\{f(\theta\mathbf{x}_t)\right\}$ in arbitrary iteration $t$ by setting $\lambda_{\theta}=+\infty$ in Alg. \ref{alg:rapid}, because there exists a close-form solution in this case. Letting $\forall t, \frac{\partial f(\theta\mathbf{x}_t)}{\partial\theta}=0$, we have $\theta_t=\frac{\mathbf{y}^T\mathbf{A}\mathbf{x}_t-f_2(\mathbf{x}_t)}{\|\mathbf{A}\mathbf{x}_t\|_2^2}$ for LASSO or group LASSO, and $\theta_t=\frac{\sum_j\mathbf{y}_j^T\mathbf{A}(\mathbf{x}_j)_t-\lambda\|\mathbf{X}_t\|_*}{\sum_j\|\mathbf{A}(\mathbf{x}_j)_t\|_2^2}$ for trace-norm.

\subsubsection{Experimental Results}
\begin{figure}[htbp]
\begin{minipage}[b]{0.33\linewidth}
 \begin{center}
 \centerline{\includegraphics[width=0.8\columnwidth]{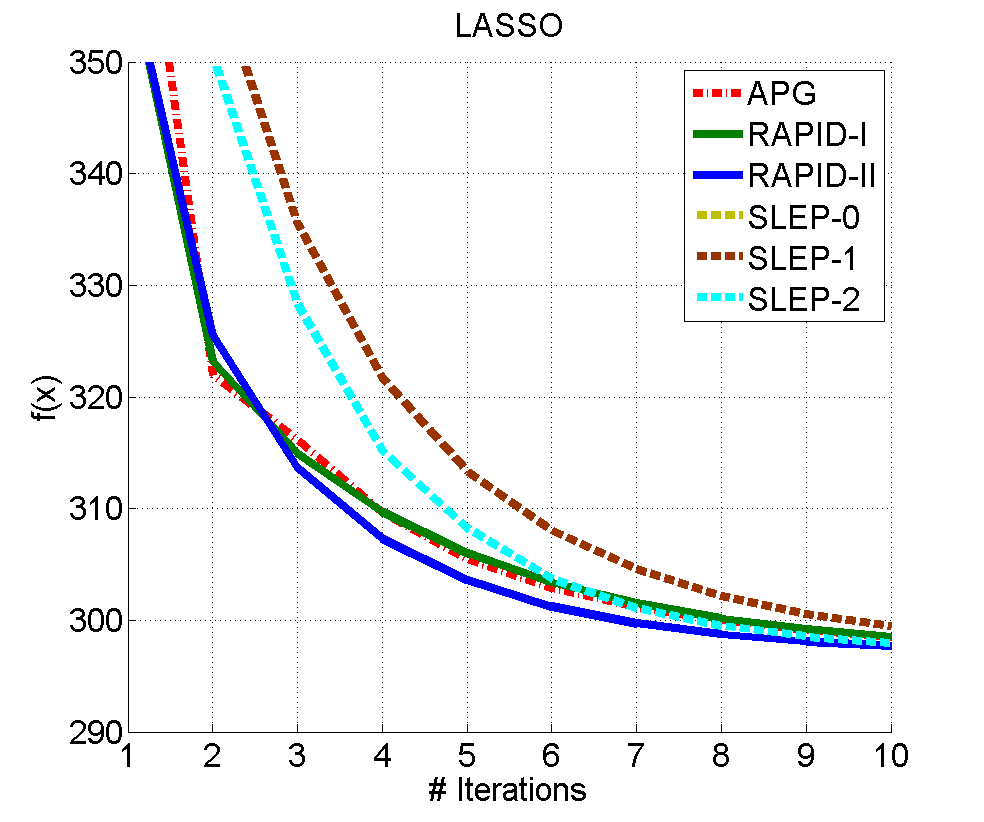}}
 \centerline{\footnotesize{(a)}}
 \end{center}
\end{minipage}
\begin{minipage}[b]{0.33\linewidth}
 \begin{center}
 \centerline{\includegraphics[width=0.8\columnwidth]{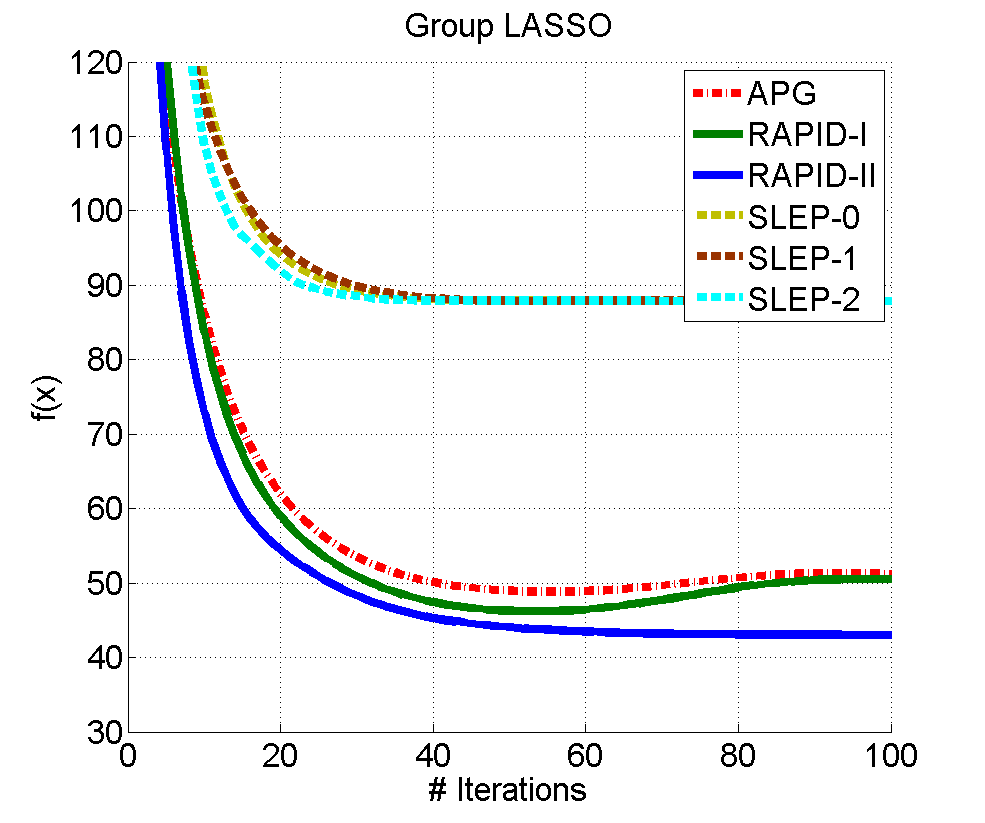}}
 \centerline{\footnotesize{(b)}}
 \end{center}
\end{minipage}
\begin{minipage}[b]{0.33\linewidth}
 \begin{center}
 \centerline{\includegraphics[width=0.8\columnwidth]{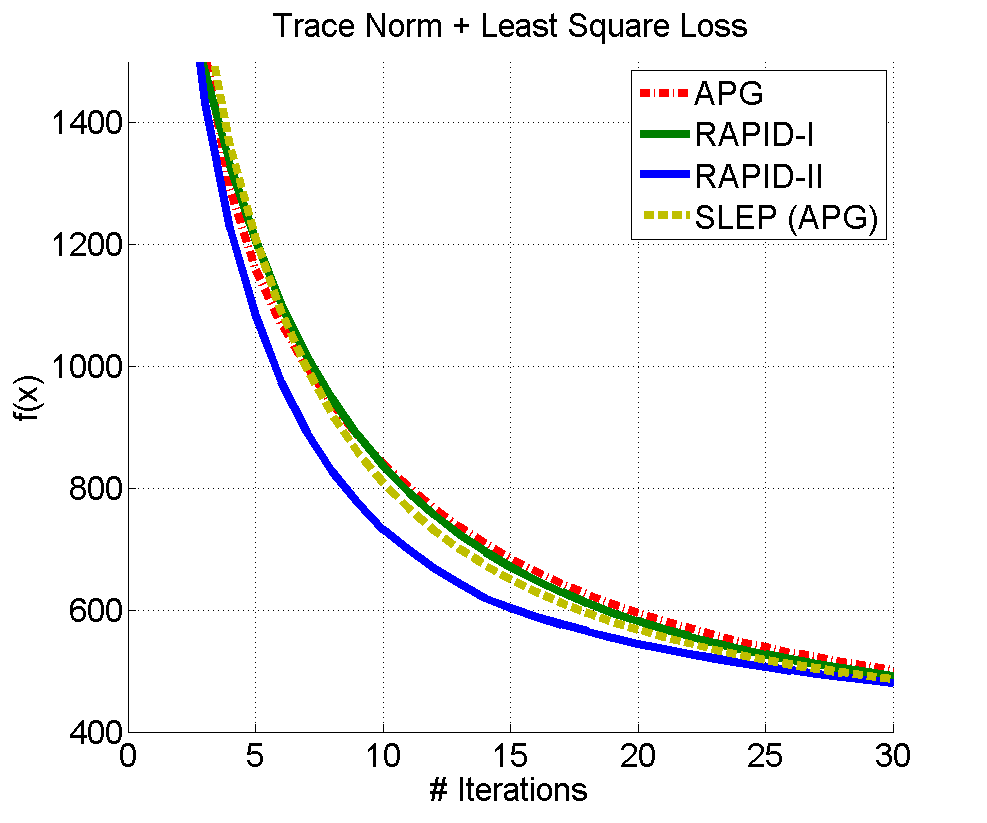}}
 \centerline{\footnotesize{(c)}}
 \end{center}
\end{minipage}\vspace{-4mm}
\caption{\footnotesize{Empirical convergence rate comparison using synthetic data on (a) LASSO, (b) group LASSO, and (c) trace-norm. This figure is best viewed in color.}}\label{fig:synthetic}
\end{figure}

We test and compare our RAPID algorithm on some synthetic data. For each sparse linear regression method, we generate a $10^3$ sample data matrix with $10^3$ dimensions per sample as variable $\mathbf{A}$, and its associated regression target vector (matrix) $\mathbf{y}$ ($\mathbf{Y}$) randomly by normal distributions. The APG and RAPID methods are modified based on the code in \cite{Brendan2013}\footnote{The code can be downloaded from \url{https://github.com/bodono/apg}. We do not use the re-starting scheme in the code.}, and SLEP \cite{Liu:2009:SLEP:manual} is a widely used sparse learning toolbox for our comparison. Here, RAPID-I and RAPID-II denote our algorithm using Eq. \ref{eqn:rapid-1} and Eq. \ref{eqn:rapid-2}, respectively. SLEP-0, SLEP-1, and SLEP-2 are the three settings used in SLEP with different parameters $mFlag$ and $lFlag$ (\ie (mFlag, lFlag)=SLEP-\#: (0,0)=0, (0,1)=1, (1,1)=2). Please refer to the toolbox manual for more details. For trace-norm, SLEP actually implements the APG algorithm as its solver.

Fig. \ref{fig:synthetic} shows our comparison results. In Fig. \ref{fig:synthetic}(a), the performances of SLEP-0 and SLEP-1 are identical, and thus there is only one curve (\ie the brown one) for both methods. Clearly, our RAPID-II algorithm works best in these three cases in terms of empirical convergence rate. Since we can only guarantee that the upper bound of the difference between the current and optimal objective values in each iteration in RAPID is no bigger than that in APG (see Section \ref{sec:analysis}), sometimes the actual objective value using RAPID may be larger than that using APG, as shown in Fig. \ref{fig:synthetic}(a). 

\subsection{Kernel SVMs}

We are interested in solving binary kernel SVMs as well, because it is a widely used {\em constrained} optimization problem. 

\subsubsection{Problem Settings}
Given a kernel matrix $\mathbf{K}\in\mathbb{R}^{N\times N}$ and a binary label vector $\mathbf{y}\in\{-1,1\}^{N}$ with $N$ samples, a binary kernel SVM can be formulated as follows:
\begin{eqnarray}\label{eqn:binary-svms}
\lefteqn{\hspace{-49mm}\min_{\boldsymbol{\alpha}}f(\boldsymbol{\alpha})=\frac{1}{2}\boldsymbol{\alpha}^T\mathbf{Q}\boldsymbol{\alpha}-\mathbf{e}^T\boldsymbol{\alpha}}\\
\begin{array}{ll}
s.t. & \forall i, 0\leq\alpha_i\leq C, \mathbf{y}^T\boldsymbol{\alpha}=0,\nonumber
\end{array}
\end{eqnarray}
where $\mathbf{Q}=\mathbf{K}\odot(\mathbf{y}\mathbf{y}^T)$, $\odot$ is the entry-wise product operator, $\mathbf{e}=\{1\}^N$ denotes a vector of 1's, and $C\geq0$ is a predefined constant.

\begin{algorithm}[t]\footnotesize
\SetAlgoLined
\SetKwInOut{Input}{Input}\SetKwInOut{Output}{Output}
\Input{$\mathbf{Q}$, $\mathbf{y}$, $C$}
\Output{$\boldsymbol{\alpha}$}
\BlankLine
$\mathbf{x}_0\leftarrow\mathbf{0}$; $\mathbf{v}_0\leftarrow\mathbf{x}_0$; $\theta_0\leftarrow 1$; $\eta_0\leftarrow 1$;\\
\For{$t=1,\cdots,T$}{
$\Delta\mathbf{v}_t\leftarrow\mathbf{Q}\mathbf{v}_{t-1}-\mathbf{e}$; $\Delta\mathbf{v}_t\leftarrow\Delta\mathbf{v}_t-\left(\frac{\mathbf{y}^T\Delta\mathbf{v}_t}{\|\mathbf{y}\|_1}\right)\mathbf{y}$; $\gamma_{\mathbf{v}}\leftarrow\frac{\Delta\mathbf{v}_t^T\mathbf{Q}\mathbf{v}_{t-1}-\mathbf{e}^T\Delta\mathbf{v}_t}{\Delta\mathbf{v}_t\mathbf{Q}\Delta\mathbf{v}_t}$; $\boldsymbol{\alpha}_t\leftarrow\mathbf{v}_{t-1}-\gamma_{\mathbf{v}}\Delta\mathbf{v}_t$; \\
\Repeat{$\forall i, 0\leq\alpha_{t,i}\leq C, \mathbf{y}^T\boldsymbol{\alpha}_t=0$}
{
$\forall i, \alpha_{t,i}\leftarrow\max\{0,\min\{C,v_{t,i}\}\}$;
$\boldsymbol{\alpha}_t\leftarrow\boldsymbol{\alpha}_t-\left(\frac{\mathbf{y}^T\boldsymbol{\alpha}_t}{\|\mathbf{y}\|_1}\right)\mathbf{y}$;
}
$\theta_t\leftarrow\min\{\frac{C}{\max_i\alpha_{t,i}},\frac{\mathbf{e}^T\boldsymbol{\alpha}_t}{\boldsymbol{\alpha}_t^T\mathbf{Q}\boldsymbol{\alpha}_t}\}$; $\eta_t=\frac{\sqrt{\eta_{t-1}^4+4\eta_{t-1}^2}-\eta_{t-1}^2}{2}$; Update $\mathbf{v}_t$ using either Eq. \ref{eqn:rapid-1} or Eq. \ref{eqn:rapid-2};
}
\Return $\boldsymbol{\alpha}\leftarrow\theta_t\boldsymbol{\alpha}_t$\;
\caption{RAPID-SVMs}\label{alg:rapid-svm}
\end{algorithm}

With RAPID, Eq. \ref{eqn:binary-svms} can be solved using Alg. \ref{alg:rapid-svm}, where each iteration contains the following 5 steps:
\begin{itemize}
\item[Step 1:] Perform line search for the step size $\gamma_{\mathbf{v}}$ in gradient descent by ignoring the constraints.
\item[Step 2:] Update $\boldsymbol{\alpha}_t$ using gradient descent with $\Delta\mathbf{v}_t$ and learned $\gamma_{\mathbf{v}}$.
\item[Step 3:] Alternatively project $\boldsymbol{\alpha}_t$ into one constraint set while fixing the other until both are satisfied.
\item[Step 4:] Update $\theta_t$ with guarantee that $\theta_t\boldsymbol{\alpha}_t$ satisfies the constraints.
\item[Step 5:] Update the auxiliary variable $\mathbf{v}_t$ as the same as in Alg. \ref{alg:rapid}.
\end{itemize}
Alg. \ref{alg:rapid-svm} can be adapted to an APG solver by fixing $\theta_t=1$ and using the update rule in FISTA for $\mathbf{v}_t$.

\subsubsection{Experimental Results}
\begin{figure}[htbp]
\begin{minipage}[b]{0.33\linewidth}
 \begin{center}
 \centerline{\includegraphics[width=\columnwidth]{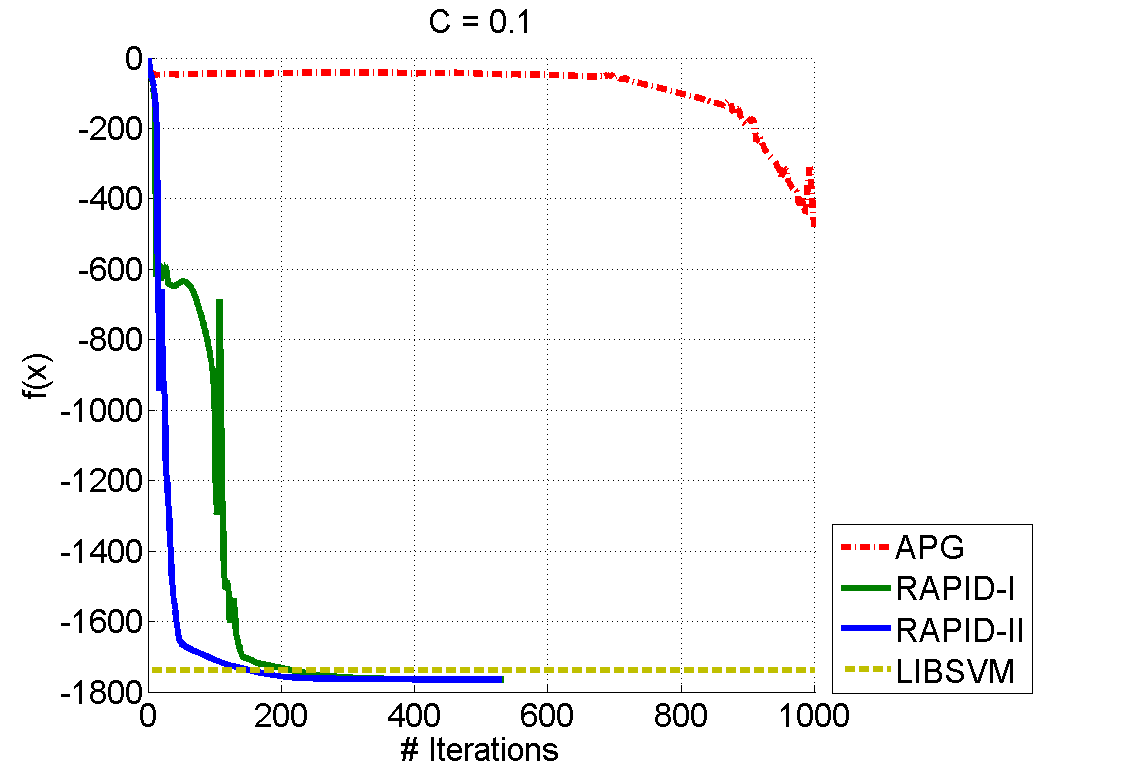}}
 \centerline{\footnotesize{(a)}}
 \end{center}
\end{minipage}
\begin{minipage}[b]{0.33\linewidth}
 \begin{center}
 \centerline{\includegraphics[width=\columnwidth]{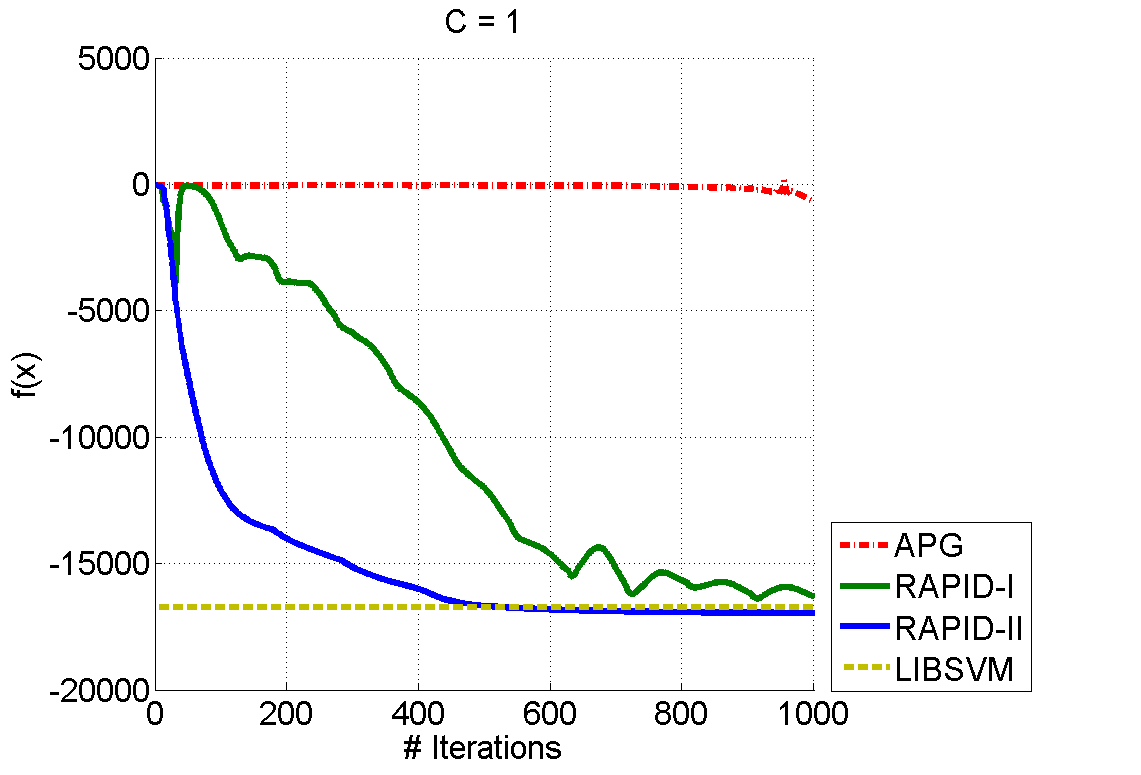}}
 \centerline{\footnotesize{(b)}}
 \end{center}
\end{minipage}
\begin{minipage}[b]{0.33\linewidth}
 \begin{center}
 \centerline{\includegraphics[width=\columnwidth]{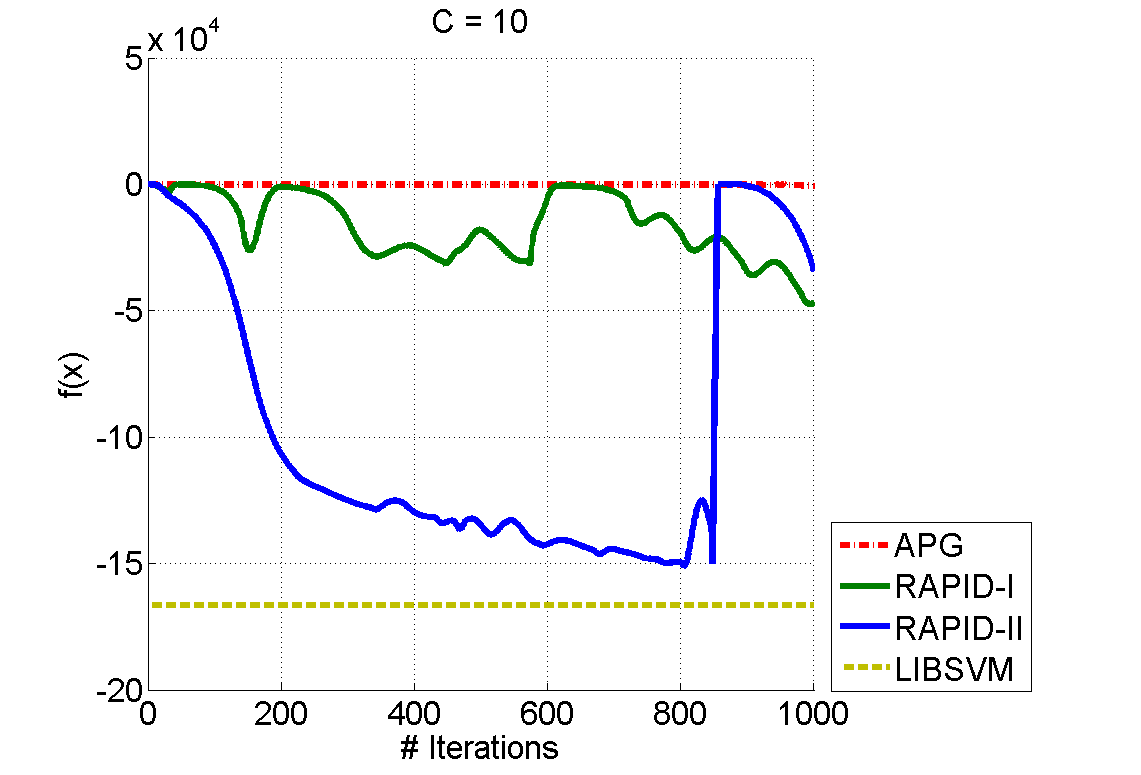}}
 \centerline{\footnotesize{(c)}}
 \end{center}
\end{minipage}\vspace{-4mm}
\caption{\footnotesize{Empirical convergence rate comparison using covtype data for SVMs.}}\label{fig:svm-comparison}
\end{figure}

We test and compare our RAPID-SVMs on the binary {\it covtype} dataset \cite{conf/nips/CollobertBB01}. This dataset contains 581012 samples with 54 dimensions per sample. Duo to the memory issue, we randomly select 5\% data as training and use the rest data as testing. For simplicity, we utilize linear kernels for creating $\mathbf{K}$, and solve Eq. \ref{eqn:binary-svms} with $C$ equal to one of $\{0.1, 1, 10\}$.

Our results are shown in Fig. \ref{fig:svm-comparison}, where RAPID is compared with APG and another popular SVM solver, LIBSVM \cite{CC01a}. The stop criterion for both RAPID is to check whether $1-\frac{\min\{|f(\theta_{t-1}\mathbf{x}_{t-1})|,|f(\theta_t\mathbf{x}_t)|\}}{\max\{|f(\theta_{t-1}\mathbf{x}_{t-1})|,|f(\theta_t\mathbf{x}_t)|\}}\leq 10^{-7}$ is satisfied. Similarly, for APG, $1-\frac{\min\{|f(\mathbf{x}_{t-1})|,|f(\mathbf{x}_t)|\}}{\max\{|f(\mathbf{x}_{t-1})|,|f(\mathbf{x}_t)|\}}\leq 10^{-7}$ is checked. For LIBSVM we use its default settings. As we see, in all these three cases RAPID converges significantly faster than APG. With the increase of $C$, RAPID begins to fluctuate. However, we can easily control this by checking the objective value in each iteration to ensure it will not increase. If increasing, the solution will not be updated. We will add this feature in our implementation in the future. Compared with LIBSVM, when $C=0.1$ or $C=1$, RAPID converges better, resulting in slightly better classification accuracies, while for $C=10$, RAPID performs worse, since it does not converge yet. Still, RAPID-II performs better than RAPID-I in all the cases.

\section{Algorithm Analysis}\label{sec:analysis}
In this section, we present our main theoretical results on the convergence rate of our RAPID algorithm in Theorem \ref{thm:rapid-1} and \ref{thm:rapid-2}, which is clearly better than those of conventional APG methods such as FISTA, leading to faster convergence in practice.
%
%

\begin{lemma}[Sandwich \cite{Baldassarre}]\label{lem:Sandwich}
Let $\tilde{f}$ be the linear approximation of $f$ in $\mathbf{v}$ \wrt $f_1$, \ie $\tilde{f}(\mathbf{w};\mathbf{v})=f_1(\mathbf{v})+\langle\triangledown f(\mathbf{v}), \mathbf{w}-\mathbf{v}\rangle+f_2(\mathbf{w})$, where $\langle\cdot,\cdot\rangle$ denotes the inner product between two vectors. Then 
\begin{equation}
f(\mathbf{w})\leq\tilde{f}(\mathbf{w};\mathbf{v})+\frac{L}{2}\|\mathbf{w}-\mathbf{v}\|_2^2\leq f(\mathbf{w})+\frac{L}{2}\|\mathbf{w}-\mathbf{v}\|_2^2.
\end{equation}
\end{lemma}

\begin{lemma}[3-Point Property \cite{Baldassarre}]\label{lem:3-point}
If $\hat{\mathbf{w}}=\argmin_{\mathbf{w}\in\mathbb{R}^d}\frac{1}{2}\|\mathbf{w}-\mathbf{w}_0\|_2^2+\phi(\mathbf{w})$, then for any $\mathbf{w}\in\mathbb{R}^d$,
\begin{equation}
\phi(\hat{\mathbf{w}})+\frac{1}{2}\|\hat{\mathbf{w}}-\mathbf{w}_0\|_2^2\leq\phi(\mathbf{w})+\frac{1}{2}\|\mathbf{w}-\mathbf{w}_0\|_2^2-\frac{1}{2}\|\mathbf{w}-\hat{\mathbf{w}}\|_2^2.
\end{equation}
\end{lemma}

\begin{lemma}
In Alg. \ref{alg:rapid}, at an arbitrary iteration $t$, we have
\begin{equation}\label{eqn:2-terms-theta}
f(\theta_t\mathbf{x}_t)\leq f(\mathbf{x}_t)-\frac{1}{2\lambda_{\theta}}\|\mathbf{x}_t-\theta_t\mathbf{x}_t\|_2^2.
\end{equation}
\end{lemma}
\begin{proof}\footnotesize
Clearly, the definition of $\theta_t$ in Alg. \ref{alg:rapid} satisfies the condition in Lemma \ref{lem:3-point}. Therefore, we have:
\begin{eqnarray}
f(\theta_t\mathbf{x}_t)&\leq& f(\theta_t\mathbf{x}_t)+\frac{1}{2\lambda_{\theta}}\|\theta_t\mathbf{x}_t-\mathbf{x}_t\|_2^2 \\
&\leq& \min_{\theta_t^*}\left\{f(\theta_t^*\mathbf{x}_t)+\frac{1}{2\lambda_{\theta}}\left(\|\theta_t^*\mathbf{x}_t-\mathbf{x}_t\|_2^2-\|\theta_t^*\mathbf{x}_t-\theta_t\mathbf{x}_t\|_2^2\right)\right\} \,\Leftarrow (\mbox{Lemma}\, \ref{lem:3-point})\nonumber\\
&\leq& f(\mathbf{x}_t)-\frac{1}{2\lambda_{\theta}}\|\mathbf{x}_t-\theta_t\mathbf{x}_t\|_2^2.\quad\quad\quad\quad\quad\quad\quad\quad\quad\quad\quad\quad\quad\Leftarrow (\theta_t^*=1)\nonumber
\end{eqnarray}
\end{proof}

\begin{thm}\label{thm:rapid-1}
Let $\mathbf{x}^*=\argmin_{\mathbf{x}\in\mathcal{X}}f(\mathbf{x})$ and $\lambda_{\theta}=\frac{1}{L}$. If Alg. \ref{alg:rapid} updates $\mathbf{v}_t$ using Eq. \ref{eqn:rapid-1}, in iteration $T (\forall  T\geq1)$ in Alg. \ref{alg:rapid}, we have
\begin{equation}
f(\theta_T\mathbf{x}_T)-f(\mathbf{x}^*)\leq\frac{2L}{(T+1)^2}\left[\|\mathbf{x}^*-\mathbf{z}_0\|_2^2-\sum_{t=1}^{T}\frac{\|\mathbf{x}_t-\theta_t\mathbf{x}_t\|_2^2}{\eta_{t-1}^2}\right],
\end{equation}
where $\mathbf{z}_0$ is a constant.
\end{thm}
\begin{proof}\footnotesize
Since $\mathbf{x}_{t+1}$ satisfies the conditions in Lemma \ref{lem:Sandwich} and Lemma \ref{lem:3-point}, we have:
\begin{eqnarray}\label{eqn:2-terms-x}
f(\mathbf{x}_{t+1})&\leq& \tilde{f}(\mathbf{x}_{t+1};\mathbf{v}_{t})+\frac{L}{2}\|\mathbf{x}_{t+1}-\mathbf{v}_{t}\|_2^2 \leq \tilde{f}(\mathbf{v}_{t}^*;\mathbf{v}_{t})+\frac{L}{2}\left(\|\mathbf{v}_{t}^*-\mathbf{v}_{t}\|_2^2-\|\mathbf{v}_{t}^*-\mathbf{x}_{t+1}\|_2^2\right)\nonumber\\
&=&\tilde{f}(\mathbf{v}_t^*; \mathbf{v}_t)+\frac{L\eta_t^2}{2}\left(\|\mathbf{x}^*-\mathbf{z}_t\|_2^2-\|\mathbf{x}^*-\mathbf{z}_{t+1}\|_2^2\right),
\end{eqnarray}
where $\mathbf{v}_{t}^*=(1-\eta_{t})\theta_t\mathbf{x}_t+\eta_{t}\mathbf{x}^*$, and $\mathbf{z}_{t+1}=\theta_{t}\mathbf{x}_t+\eta_{t}^{-1}(\mathbf{x}_{t+1}-\theta_t\mathbf{x}_t)$. Due to the convexity of $\tilde{f}$, we can rewrite Eq. \ref{eqn:2-terms-x} as follows:
\begin{eqnarray}\label{eqn:theta_tx}
f(\mathbf{x}_{t+1})&\leq&(1-\eta_t)\tilde{f}(\theta_t\mathbf{x}_t;\mathbf{v}_t)+\eta_t\tilde{f}(\mathbf{x}^*;\mathbf{v}_t)+\frac{L\eta_t^2}{2}\left(\|\mathbf{x}^*-\mathbf{z}_t\|_2^2-\|\mathbf{x}^*-\mathbf{z}_{t+1}\|_2^2\right)\nonumber\\
&\leq&(1-\eta_t)f(\theta_t\mathbf{x}_t)+\eta_t f(\mathbf{x}^*)+\frac{L\eta_t^2}{2}\left(\|\mathbf{x}^*-\mathbf{z}_t\|_2^2-\|\mathbf{x}^*-\mathbf{z}_{t+1}\|_2^2\right).
\end{eqnarray}
Based on Eq. \ref{eqn:2-terms-theta} and \ref{eqn:theta_tx}, we have:
\begin{equation}\label{eqn:key}
f(\theta_{t+1}\mathbf{x}_{t+1})\leq(1-\eta_t)f(\theta_t\mathbf{x}_t)+\eta_t f(\mathbf{x}^*)+\frac{L\eta_t^2}{2}\left(\|\mathbf{x}^*-\mathbf{z}_t\|_2^2-\|\mathbf{x}^*-\mathbf{z}_{t+1}\|_2^2\right)-\frac{L}{2}\|\mathbf{x}_{t+1}-\theta_{t+1}\mathbf{x}_{t+1}\|_2^2.
\end{equation}
Letting $\epsilon_t=f(\theta_t\mathbf{x}_t)-f(\mathbf{x}^*)$ and $\Phi_t=\frac{L}{2}\|\mathbf{x}^*-\mathbf{z}_t\|_2^2$, we can rewrite Eq. \ref{eqn:key} as follows:
\begin{eqnarray}
&&\epsilon_{t+1}\leq(1-\eta_t)\epsilon_t+\eta_t^2(\Phi_t-\Phi_{t+1})-\frac{L}{2}\|\mathbf{x}_{t+1}-\theta_{t+1}\mathbf{x}_{t+1}\|_2^2\\
&\Leftrightarrow& \frac{1}{\eta_t^2}\epsilon_{t+1}-\frac{1-\eta_t}{\eta_t^2}\epsilon_t\leq\Phi_t-\Phi_{t+1}-\frac{L}{2\eta_t^2}\|\mathbf{x}_{t+1}-\theta_{t+1}\mathbf{x}_{t+1}\|_2^2\nonumber
\end{eqnarray}
$\because$ the sequence $\{\eta_t\}_{t=0,1,\cdots}$ in Alg. \ref{alg:rapid} satisfies $\forall t, \frac{1-\eta_{t+1}}{\eta_{t+1}^2}=\frac{1}{\eta_t^2}$, $\therefore \frac{1-\eta_{t+1}}{\eta_{t+1}^2}\epsilon_{t+1}=\frac{1}{\eta_{t}^2}\epsilon_{t+1}$, leading to
\begin{eqnarray}
\frac{1-\eta_{t+1}}{\eta_{t+1}^2}\epsilon_{t+1}-\frac{1-\eta_t}{\eta_t^2}\epsilon_t\leq\Phi_t-\Phi_{t+1}-\frac{L}{2\eta_t^2}\|\mathbf{x}_{t+1}-\theta_{t+1}\mathbf{x}_{t+1}\|_2^2.
\end{eqnarray}
Letting $g_t=\frac{1-\eta_t}{\eta_t^2}\epsilon_t$, we have $g_0=0$ and
\begin{eqnarray}\label{eqn:deduction}
\lefteqn{\hspace{-16mm}\forall T\geq1, \, \sum_{t=0}^{T-1}\left(g_{t+1}-g_t\right)=g_T-g_0=\frac{\epsilon_T}{\eta_{T-1}^2}\leq\sum_{t=0}^{T-1}\left(\Phi_t-\Phi_{t+1}\right)-\frac{L}{2}\sum_{t=0}^{T-1}\frac{\|\mathbf{x}_{t+1}-\theta_{t+1}\mathbf{x}_{t+1}\|_2^2}{\eta_t^2}}\nonumber\\
&& \leq\Phi_0-\frac{L}{2}\sum_{t=1}^{T}\frac{\|\mathbf{x}_t-\theta_t\mathbf{x}_t\|_2^2}{\eta_{t-1}^2}=\frac{L}{2}\|\mathbf{x}^*-\mathbf{z}_0\|_2^2-\frac{L}{2}\sum_{t=1}^{T}\frac{\|\mathbf{x}_t-\theta_t\mathbf{x}_t\|_2^2}{\eta_{t-1}^2}.
\end{eqnarray}
Since the sequence $\{\eta_t\}_{t=0,1,\cdots}$ also satisfies $\eta_t\leq \frac{2}{t+2}$, based on Eq. \ref{eqn:deduction} we have $\frac{\epsilon_T(T+1)^2}{4}\leq\frac{\epsilon_T}{\eta_{T-1}^2}$. Therefore,
\begin{eqnarray}\label{eqn:conclusion}
\epsilon_T\leq\frac{2L}{(T+1)^2}\left[\|\mathbf{x}^*-\mathbf{z}_0\|_2^2-\sum_{t=1}^{T}\frac{\|\mathbf{x}_t-\theta_t\mathbf{x}_t\|_2^2}{\eta_{t-1}^2}\right].
\end{eqnarray}
\end{proof}

\begin{thm}\label{thm:rapid-2}
Let $\mathbf{x}^*=\argmin_{\mathbf{x}\in\mathcal{X}}f(\mathbf{x})$ and $\lambda_{\theta}=\frac{1}{L}$. If Alg. \ref{alg:rapid} updates $\mathbf{v}_t$ using Eq. \ref{eqn:rapid-2}, and suppose in any iteration $t (t\geq1)$,  $\mathbf{v}_t^*=(1-\eta_t)\theta_t\mathbf{x}_t+\eta_t\mathbf{x}^*$ and $\|\mathbf{v}_t^*-\mathbf{x}_{t+1}\|_2^2=\|\mathbf{v}_t^*-\theta_{t+1}\mathbf{x}_{t+1}\|_2^2+\xi_{t+1}$, then in iteration $T (\forall T\geq1)$ we have
\begin{equation}
f(\theta_T\mathbf{x}_T)-f(\mathbf{x}^*)\leq\frac{2L}{(T+1)^2}\left[\|\mathbf{x}^*-\mathbf{z}_0\|_2^2-\sum_{t=1}^{T}\left(\frac{\|\mathbf{x}_t-\theta_t\mathbf{x}_t\|_2^2}{\eta_{t-1}^2}+\xi_t\right)\right],
\end{equation}
where $\mathbf{z}_0$ is a constant.
\end{thm}
\begin{proof}\footnotesize
Based on the assumptions of the theorem and Eq. \ref{eqn:2-terms-x}, we have
\begin{eqnarray}\label{eqn:2-terms-x-2}
f(\mathbf{x}_{t+1})
&\leq& \tilde{f}(\mathbf{v}_{t}^*;\mathbf{v}_{t})+\frac{L}{2}\left(\|\mathbf{v}_{t}^*-\mathbf{v}_{t}\|_2^2-\|\mathbf{v}_{t}^*-\mathbf{x}_{t+1}\|_2^2\right)\nonumber\\
&=& \tilde{f}(\mathbf{v}_{t}^*;\mathbf{v}_{t})+\frac{L}{2}\left(\|\mathbf{v}_{t}^*-\mathbf{v}_{t}\|_2^2-\|\mathbf{v}_{t}^*-\theta_{t+1}\mathbf{x}_{t+1}\|_2^2-\xi_{t+1}\right)\nonumber\\
&=&\tilde{f}(\mathbf{v}_t^*; \mathbf{v}_t)+\frac{L\eta_t^2}{2}\left(\|\mathbf{x}^*-\mathbf{z}_t\|_2^2-\|\mathbf{x}^*-\mathbf{z}_{t+1}\|_2^2-\xi_{t+1}\right),
\end{eqnarray}
where $\mathbf{z}_{t+1}=\theta_{t}\mathbf{x}_t+\eta_{t}^{-1}(\theta_{t+1}\mathbf{x}_{t+1}-\theta_t\mathbf{x}_t)$. Following the same proof strategy for Theorem \ref{thm:rapid-1}, we can easily prove this theorem.
\end{proof}

\section{Conclusion}\label{sec:con}

In this paper, we propose an improved APG algorithm, namely, Rapidly Accelerated Proximal Gradient (RAPID), to speed up the convergence of conventional APG algorithms. Our first idea is to introduce a new line search step after the proximal gradient step in APG to push the current solution $\mathbf{x}_t$ towards a new one $\theta\mathbf{x}_t\in\mathcal{X} (\theta>0)$ so that $f(\theta\mathbf{x}_t)$ is minimized over scalar $\theta$. Our second idea is to propose two different ways of constructing the auxiliary variable in APG using the intermediate solutions in the previous and current iterations. In this way, we can prove that our algorithm is guaranteed to converge with a smaller upper bound of the gap between the current and optimal objective values than those in APG algorithms. We demonstrate our algorithm using two applications, \ie sparse linear regression and kernel SVMs. In summary, our RAPID converges faster than APG, in general, and for some problems RAPID based algorithms can be comparable with the sophisticated existing solvers.

\bibliographystyle{ieee}
\bibliography{egbib}

\end{document}